\documentclass[11pt]{article}

\usepackage[preprint]{acl}

\usepackage{times}
\usepackage{latexsym}

\usepackage[T1]{fontenc}

\usepackage[utf8]{inputenc}

\usepackage{microtype}

\usepackage{inconsolata}

\usepackage{graphicx}

\usepackage{tikz}
\usepackage{algorithm}
\usepackage{algpseudocode}
\usepackage{booktabs}
\usepackage{colortbl}
\usepackage{xcolor}

\definecolor{barblue}{RGB}{70, 130, 180} 
\definecolor{bargray}{RGB}{220, 220, 220} 
\definecolor{ourbg}{RGB}{235, 245, 255}   
\usepackage{times}
\usepackage{latexsym}
\usepackage{enumitem}
\usepackage[most]{tcolorbox}

\usepackage{amsmath}
\usepackage{amssymb}
\usepackage{amsthm}

\newcommand{\cbar}[2]{%
    \begin{tikzpicture}[baseline={(0,0.08)}] 
        \fill[bargray, rounded corners=2pt] (0,0) rectangle (1.4, 0.35); 
        \fill[barblue, rounded corners=2pt] (0,0) rectangle (#1*1.4, 0.35); 
        \node[right, font=\scriptsize\sffamily] at (1.5, 0.18) {#2}; 
    \end{tikzpicture}%
}

\newcommand{\level}[1]{%
    \begin{tikzpicture}[baseline=-0.3em]
        \foreach \i in {1,...,5} {
            \fill[black!10] (\i*0.25, 0) circle (0.06); 
        }
        \foreach \i in {1,...,#1} {
            \fill[barblue] (\i*0.25, 0) circle (0.06); 
        }
    \end{tikzpicture}%
}

\theoremstyle{plain}
\newtheorem{theorem}{Theorem}[section]

\theoremstyle{definition}
\newtheorem{definition}[theorem]{Definition}
\newtheorem{assumption}[theorem]{Assumption}
\theoremstyle{remark}

\usepackage[utf8]{inputenc}
\usepackage{booktabs} 
\usepackage{multirow} 
\usepackage{xcolor}   
\usepackage{colortbl} 

\newcommand{\gain}[1]{\textsubscript{\textbf{\color{green!50!black}{+#1}}}}
\definecolor{rowblue}{RGB}{235, 245, 250}
\usepackage{bm}

\usepackage[most]{tcolorbox}
\usepackage{xcolor}
\definecolor{promptbg}{RGB}{248,248,248}
\definecolor{promptborder}{RGB}{190,190,190}
\newtcolorbox{promptbox}[2][]{
  enhanced,
  breakable,
  colback=promptbg,
  colframe=promptborder,
  boxrule=0.6pt,
  arc=2pt,
  left=8pt,
  right=8pt,
  top=6pt,
  bottom=6pt,
  fonttitle=\bfseries,
  title={#2},
  coltitle=black,
  #1
}

\title{ToolVerse: Unlocking Massive Environments and \\ Long-Horizon Tasks for Agentic Reinforcement Learning}

\author{
  \textbf{Shuaiyu Zhou\textsuperscript{1,2}\thanks{These authors contributed equally and should be considered co-first authors.}} 
  \quad
  \textbf{Fengpeng Yue\textsuperscript{1}\footnotemark[1]} 
  \quad
  \textbf{Zengjie Hu\textsuperscript{1,2}\footnotemark[1]}
  \quad
  \textbf{Yuanzhe Shen\textsuperscript{1,3}}
  \\[4pt]
  \textbf{Chenyang Zhang\textsuperscript{1,4}}
  \quad
  \textbf{Feng Hong\textsuperscript{1}}
  \quad
  \textbf{Cao Liu\textsuperscript{1}}
  \quad
  \textbf{Ke Zeng\textsuperscript{1}}
  \\[6pt]
  \textsuperscript{1}LongCat Interaction Team, Meituan, Beijing, China
  \\
  \textsuperscript{2}Peking University, Beijing, China
  \quad
  \textsuperscript{3}Fudan University, Shanghai, China
  \quad
  \textsuperscript{4}Wuhan University, Wuhan, China
  \\[6pt]
  \{monster290,huzengjie\}@stu.pku.edu.cn, yzshen25@m.fudan.edu.cn, chenyoung@whu.edu.cn
  \\
  \{yuefengpeng02,hongfeng03,liucao,zengke\}@meituan.com
}

\begin{document}
\maketitle
\begin{abstract}
While LLM agents demonstrate strong reasoning abilities in compact and well-defined scenarios, they struggle to maintain robustness and effectiveness when faced with large-scale, diverse, and dynamic real-world environments that demand seamless tool integration. To address this gap, we introduce \textbf{ToolVerse}, a comprehensive  framework that scales up agentic RL environments and enables agents to perform complex long-horizon reasoning in Tool-Integrated Reasoning (TIR) tasks. First, ToolVerse automatically builds the massive executable agent training environments from nearly 422 real-world MCP environments that contain about 4438 tools. Second, we propose a task design strategy based on a tool dependency graph, utilizing  Dynamic Unlocking Sampling Algorithm  to generate long-horizon tasks, and produce \textbf{GUST} (\textbf{G}raph \textbf{U}nlocking \textbf{S}ampling \textbf{T}asks) dataset. Third, to alleviate the credit assignment problem in long-horizon agentic RL, we propose a fine-grained Turn-Aware Relative Advantage algorithm. We conduct extensive Agentic RL training using ToolVerse and evaluate our framework on several agentic benchmarks. Experimental results demonstrate that our framework significantly strengthens LLMs' capabilities in long-horizon tool use, achieving a marked performance boost and showcasing robust reasoning within dynamic environments.
\end{abstract}

\section{Introduction}
Large Language Models (LLMs) have demonstrated considerable promise as autonomous agents, capable of interacting with tools and environments to accomplish complex tasks \cite{webdancer, ye2025feedbackdriventooluseimprovementslarge, apigen-mt}. When combined with Agentic Reinforcement Learning, LLMs can develop the ability to reason over long horizons and make sequential decisions \cite{rstar2agent, toolrl, toolstar, torl, skyrlagent}. By integrating tool usage with reasoning and adapting actions based on intermediate observations, LLMs can plan over extended time frames, positioning them as powerful candidates for autonomous agent development in dynamic, real-world environments.

\begin{table*}[t]
\setlength{\tabcolsep}{3pt}

    \begin{center}
        \begin{small}
        \caption{We posit the scalability of \textbf{Agentic RL} is primarily limited by three dimensions: \textbf{Scope} (environment diversity), \textbf{Tool Use Complexity} (reasoning depth), and \textbf{Credit Assignment} (training signal granularity). ToolVerse bridges the gap across these aspects.}
        \label{tab:visual_comparison}
        
        \begin{tabular}{l l l l}
        \toprule
        \textbf{Method} & \textbf{Tool Scope (Env Diversity)} & \textbf{Tool Use Complexity} & \textbf{Credit Assignment} \\
        \midrule
        ZeroTIR     & \cbar{0.2}{Single tool (Code)}   & \level{1} \footnotesize Math Reasoning & Outcome (Sparse) \\
        SimpleTIR   & \cbar{0.2}{Single tool (Code)}   & \level{1} \footnotesize Math Reasoning & Outcome (Sparse) \\
        AGENTFLOW   & \cbar{0.4}{Search + Code}        & \level{2} \footnotesize QA/Agentic/Math Reasoning & Outcome (Sparse) \\
        ToolRL      & \cbar{0.6}{Three Benchmarks' tools}         & \level{3} \footnotesize Tool use Reasoning & Outcome (Dense) \\
        SALT        & \cbar{0.6}{Three Benchmarks' tools}         & \level{3} \footnotesize Tool use Reasoning & Outcome (Sparse) \\
        FTRL        & \cbar{1.0}{Many General tools}         & \level{4} \footnotesize QA/Tool Use Reasoning & Outcome (Sparse) \\
        \midrule
        \rowcolor{ourbg}
        \textbf{ToolVerse (Ours)} & \cbar{1.0}{Many General tools} & \level{5} \textbf{Long-horizon Tool Reasoning } & \textbf{Turn-Aware (Dense)} \\
        \bottomrule
        \end{tabular}
        \end{small}
    \end{center}
\end{table*}

However, as shown in Figure \ref{tab:visual_comparison}, developing agentic systems presents three primary challenges. First, existing Agent RL environments are typically limited to interactions with a single tool or a small set of tools, such as search engines or code interpreters \cite{zerotir, li2025intheflowagenticoptimizationeffective, toolrl, simpletir}. While effective within narrow domains, these environments lack the complexity required for long-horizon, multi-turn tool integration. Second, designing tasks that involve multi-turn, long-horizon reasoning across multiple integrated tools is inherently difficult \cite{ye2025feedbackdriventooluseimprovementslarge}. These tasks demand that each action be informed by prior steps, requiring sophisticated planning and reasoning over extended sequences of interactions. Third, precise credit assignment is difficult in multi-turn agentic RL. Sparse terminal rewards fail to provide meaningful advantage estimates for individual actions within a long-horizon trajectory. This lack of granular feedback leads to high variance in policy gradients, resulting in poor sample efficiency or even training instability.

In this paper, we introduce \textbf{ToolVerse}, a comprehensive framework designed to address the limitations of current Agentic RL training environments and to support more complex agent training. \textbf{(1) Scaling Executable Agentic RL Environments}: To overcome the lack of diverse, general-purpose environments, we scale the training landscape by automating the generation of executable MCP tools. This expands the range of tools available for interaction, facilitating training in more realistic, complex scenarios. \textbf{(2) Graph-Guided Long-Horizon Task Composition}: To address the challenge of task complexity, we propose a task design strategy based on a tool dependency graph and a dynamic unlocking sampling algorithm, producing \textbf{GUST} dataset. This approach generates task subsets that require integrated reasoning across multiple steps, with outputs from earlier tasks feeding into subsequent ones, thereby enabling agents to develop sophisticated tool use capabilities over extended interaction sequences. \textbf{(3) Turn-Aware Relative Advantage Estimation}: To enhance credit assignment during trajectory evaluation, we introduce turn-aware reward and advantage estimation. This allows the model to learn from training with more precise feedback, enabling it to develop a deeper understanding from the task performance at each step.

Our contributions can be summarized as follows:

\begin{itemize}[nosep, leftmargin=1.5em, itemsep=0pt]
    \item We develop an automated pipeline for scaling executable MCP environments, expanding agentic RL from narrow domains to diverse, real-world tool-integrated reasoning scenarios.
    
    \item We propose a TDG-based task generation paradigm with the \textbf{Dynamic Unlocking Sampling Algorithm}, producing coherent long-horizon and multi-turn tasks and yielding the \textbf{GUST} dataset for agentic RL training.
    
    \item We introduce Turn-Aware Relative Advantage estimation for fine-grained credit assignment under sparse rewards, improving training stability and tool-integrated reasoning performance across agentic benchmarks.
\end{itemize}

\begin{figure*}[!t]  
    \centering  
    \includegraphics[width=1\textwidth]{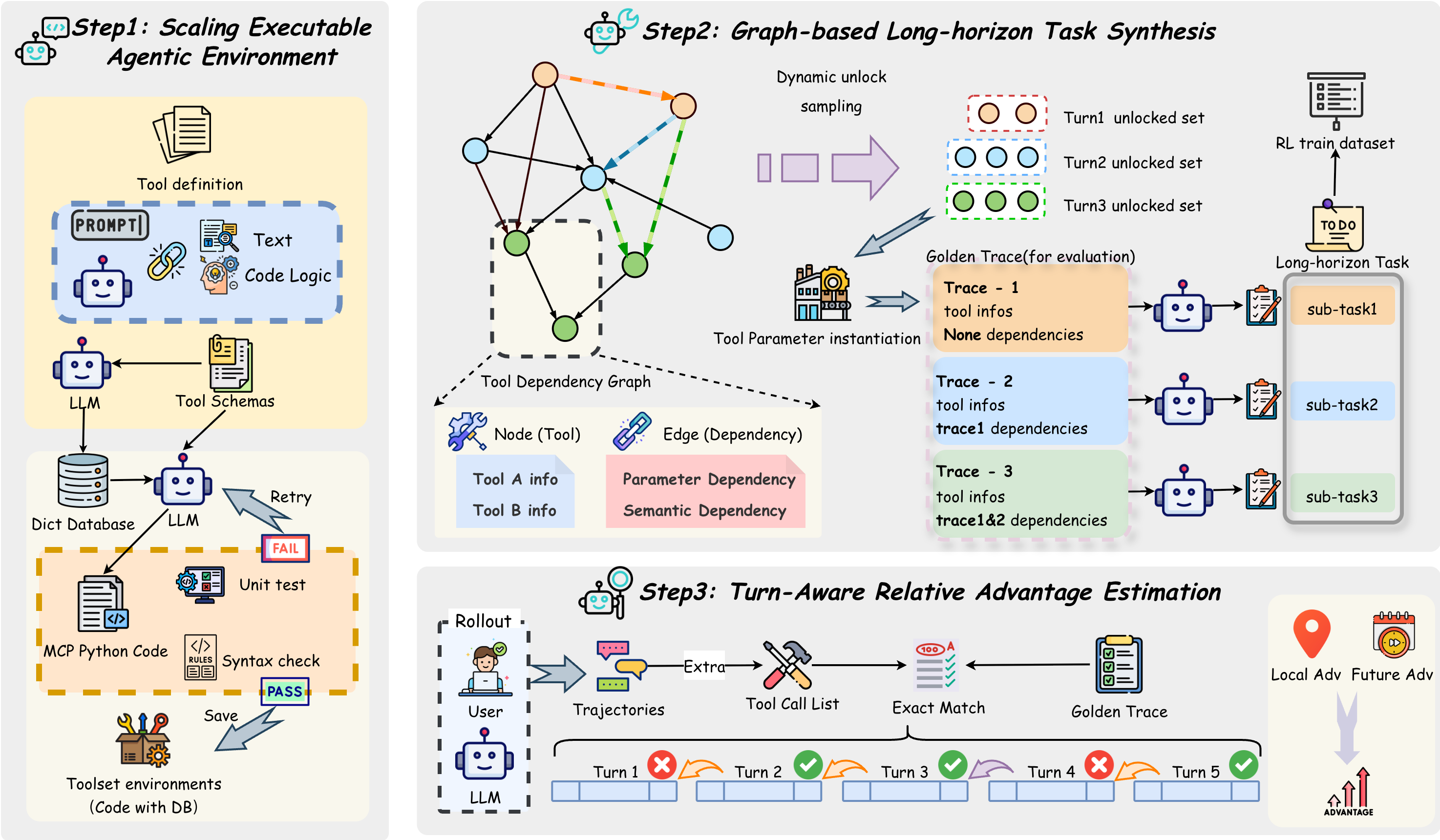}
    \caption{Overview of the \textbf{ToolVerse} framework. \textit{\textbf{Step1: Scaling Executable Agentic Environment}}: We automate the conversion of raw tool definitions into executable MCP environments. \textit{\textbf{Step2: Graph-based Long-horizon Task Synthesis}}: A tool dependency graph is constructed to model logical data flows. We employ a dynamic unlocking sampling algorithm to synthesize long-horizon tasks. \textit{\textbf{Step3: Turn-Aware Relative Advantage Estimation}}: We propose turn-aware advantage algorithms based on the characteristics of the task.}
    \label{fig:toolverse_framework} 
\end{figure*}

\section{Related Work}
\subsection{Tool-integrated reasoning in LLM Agents} 
Recent advancements in Agentic RL have highlighted the significance of integrating tool usage within large language models (LLMs) to enhance their ability to perform complex, multi-step reasoning. 
For instance, TORL \cite{torl} and rStar2-Agent \cite{rstar2agent} leverage reinforcement learning to train LLMs in utilizing code interpreters, treating code execution as a reliable feedback mechanism. 
Similarly, ZeroTIR \cite{zerotir} investigates scaling laws for agentic RL, revealing that tool-use abilities can emerge through outcome-based rewards, with performance linked to the computational resources available. 
To maintain stability in multi-turn training, SimpleTIR \cite{simpletir} filters out “void” trajectories—sequences that lack meaningful tool use or outputs. 
Additionally, multi-tool collaboration has been explored by works like ToolStar\cite{toolstar} , which enable agents to coordinate between search engines and code interpreters. ToolRL \cite{toolrl} focuses on reward design to optimize selection strategies across multiple tools.
However, these efforts are predominantly constrained to narrow toolsets, such as code interpreters and search engines, and typically involve fixed interaction patterns. 
This limitation hinders the development of generalized capabilities for interacting with diverse, structured APIs in real-world environments.


\subsection{Training environment for Agentic RL}

The progress of Agentic RL is fundamentally dependent on the fidelity and scalability of its training environments.
Several studies \cite{li2025simulating, castellani2025synthtools, xiang2023language} utilize LLMs’ reasoning abilities and background knowledge to emulate environments.
While this approach eliminates the necessity of constructing actual environments, it often suffers from issues such as hallucinations, inconsistency, limited transparency, and inadequate management of persistent states.
Recent methods such as TaskCraft \cite{taskcraft} have automated the generation of scalable tasks to address data scarcity, while offline-database methods \cite{fang2025generalagenticintelligenceenvironment, ye2025feedbackdriventooluseimprovementslarge} utilize static captures of real-world data to preserve simulation fidelity.


\section{Method}
\vspace{-0.2em}

In this section, we delineate the \textbf{ToolVerse} framework. We first introduce the automated synthesis pipeline for scaling executable environments (Sec.~\ref{sec:scaling}). We then formalize the graph-guided generation mechanism for long-horizon task curricula (Sec.~\ref{sec:graph}). Finally, we use a Turn-Aware Relative Advantage Estimation algorithm (Sec.~\ref{sec:reward}) tailored for long-horizon reasoning, addressing the credit assignment limitations in standard group-based reinforcement learning.

\subsection{Scaling Executable Agentic Environments}
\label{sec:scaling}
To construct agentic environments across diverse domains, we propose an automated closed-loop synthesis pipeline. Starting with a repository of raw toolsets $\mathcal{S}^{\text{raw}} = \{ S_1, \dots, S_N \}$, where each $S_i$ targets a specific functional domain, our objective is to transform these static definitions into robust executable environments.

First, we perform \textit{Schema Refactoring} to refine each raw schema $S_i$ into a structured form $\hat{S}_i$. This step aligns function signatures with intended functionality and eliminates textual noise, simplifying logic to facilitate LLM-based Python code generation. Next, we construct a domain-specific dictionary database $\mathcal{D}_i$ for each tool. This database models essential state variables (e.g., inventory, profiles) to enable stateful environmental interactions. Given $\hat{S}_i$ and $\mathcal{D}_i$, we generate executable functions $\mathcal{F}_i = \{f_1, \dots, f_m\}$ using the MCP tools library. These functions implement the logic to query or update $\mathcal{D}_i$, establishing a dynamic, operable environment. Finally, to ensure robustness, we automatically generate unit tests $\mathcal{T}_i$. We retain only toolsets that pass all syntax and unit tests ($\mathbb{V}(\mathcal{F}_i, \mathcal{T}_i) = \text{True}$). These environments are integrated into the Verl framework \cite{sheng2024hybridflow} to provide a diverse, high-fidelity action space for RL training. More construction details are provided in Appendix~\ref{app:environment_construction}

\subsection{Graph-based Long-horizon Task Synthesis}
\label{sec:graph}
To facilitate long-horizon reasoning, we move beyond random walk-based task generation and propose a structured, graph-theoretic approach to task synthesis. Before applying DUS, we sanitize the LLM-induced TDG by removing cyclic and non-executable dependencies; details are provided in Appendix~\ref{app:tdg_sanitization}. The detailed pseudocode is in Algorithm~\ref{alg:dynamic_unlocking}

\begin{figure*}[!t]  
    \centering  
    \includegraphics[width=1\textwidth]{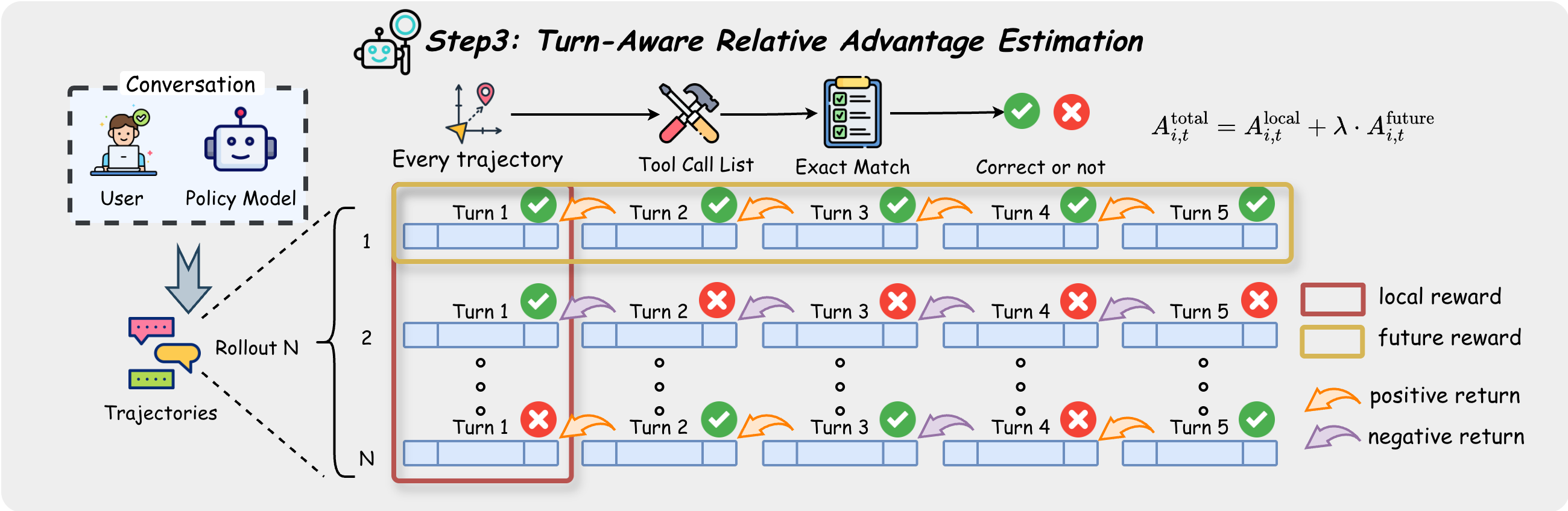}
    \caption{Overview of the \textbf{Turn-Aware Relative Advantage Estimation}. For each turn in the multi-turn multi-step tool usage trajectory, we perform rule-based validation to ensure its validity. During the relative advantage computation, we calculate \textbf{the normalized advantage for each turn} based on the group distribution at that turn, and then propagate this advantage uniformly to all tokens within the corresponding turn.}
    \label{fig:toolverse_adv} 
\end{figure*}

\subsubsection{Tool Dependency Graph Construction}
We define a tool dependency graph (TDG) for each scenario $\mathcal{G} = (\mathcal{V}, \mathcal{E})$, where nodes represent tools $\mathcal{V}$ and edges $\mathcal{E}$ capture the dependencies between them. The dependency structure is inferred with the help of an LLM based on two key principles: (1) an edge is established from tool $T_A$ to tool $T_B$ if the output of $T_A$ is required as an input for $T_B$; (2) an edge is drawn from $T_A$ to $B$ if $T_B$ can only be invoked after $T_A$ in the logical sequence of the scenario. The resulting graph, $\mathcal{G}$, serves as the foundational structure for reasoning, where paths represent valid sequences of tool invocations with contextual dependencies.

\subsubsection{Dynamic Unlocking Sampling}
\textbf{Causal Grounding via Topological Constraints.} To generate multi-turn trajectories with logical consistency, we introduce the Dynamic Unlocking Sampling (DUS) algorithm. The algorithm maintains a ready queue $\mathcal{Q} \subseteq \mathcal{V}$, consisting of tools with zero in-degree, ensuring that tasks with high dependency are “locked” until their prerequisites are completed. This topological constraint prevents causal errors in reasoning chains.

\textbf{Iterative Unlocking Mechanism.} The process starts with a pool of root nodes $\mathcal{Q} = \{v \in \mathcal{V} \mid d_{\text{in}}(v) = 0\}$, typically comprising basic operations. At each step $t$, we sample a subset $S_t \subseteq \mathcal{Q}$ to form a trajectory stage. The batch size is determined by $k = \min(|\mathcal{Q}|, N)$. After executing $S_t$, the algorithm updates the in-degree for all successor nodes: $d_{\text{in}}(v) \leftarrow d_{\text{in}}(v) - 1$. A node $v$ is then added to $\mathcal{Q}$ when $d_{\text{in}}(v) = 0$, signaling that all dependencies are satisfied.

\textbf{Emergence of Long-Horizon Logic.} This topological progression induces a curriculum of increasing task complexity within the final trajectory $T = [S_1, S_2, \dots, S_m]$. Simple queries necessarily precede more complex operations, leading to the natural emergence of complex reasoning at later stages. The resulting trajectories exhibit a coherent cognitive flow, providing high-fidelity training signals for long-horizon LLM agents.

\subsubsection{Inverse Context Reconstruction}
After DUS samples a dependency-compatible tool skeleton, we ground it in the mock database to obtain an executable \textbf{Golden Trace}. Specifically, arguments are instantiated in topological order: dependency-based arguments are filled from previous tool outputs according to TDG edges, while context-specific arguments are retrieved from the current database state.

For each environment, we sample multiple such traces, ensuring that the resulting task set exercises as many tools in the environment as possible. Given each Golden Trace, we then prompt an LLM to convert each Golden Trace into a user-facing task. The involved dependency edges are provided during generation so that the task preserves the intended parameter and semantic dependencies. Finally, we replay the generated task set with LangGraph to verify that the Golden Trace remains executable under stateful database updates, and further apply teacher-agent Pass@8 filtering. More details are provided in Appendix~\ref{app:inverse_context_reconstruction}.

\subsection{Turn-Aware Relative Advantage Estimation}
\label{sec:reward}

Standard Group Relative Policy Optimization (GRPO)~\cite{grpo} normalizes rewards across the entire trajectory, assigning a single scalar advantage to all tokens in a response. However, in long-horizon tool use, this coarse-grained feedback fails to distinguish between correct intermediate steps and fatal errors in later stages. Leveraging the structured nature of our tasks—where the ``Golden Trace'' provides ground truth for each turn—we propose a fine-grained Turn-Aware Relative Advantage mechanism.

We decompose the turn-level advantage into two components: \textit{Local} (immediate correctness) and \textit{Future} (downstream impact). For a group of $K$ outputs in a $T$-turn task, we define the total advantage $A_{i,t}^{\text{total}}$ at turn $t$ for the $i$-th agent as follows.




\paragraph{Binary Reward.}
Given the golden trace, we compute a binary reward for each turn based on tool-call coverage. Let $G_t$ be the set of golden tool JSON objects required at turn $t$, and $A_{i,t}$ be the tool-call JSON sequence generated by the $i$-th rollout. We set $r_{i,t}=1$ if $A_{i,t}$ covers $G_t$, and $r_{i,t}=0$ otherwise:
\begin{equation}
r_{i,t} =
\begin{cases}
1.0, & \text{if } G_t \subseteq_{\mathrm{dict}} A_{i,t}, \\
0.0, & \text{otherwise}.
\end{cases}
\end{equation}
Here, $\subseteq_{\mathrm{dict}}$ denotes coverage under dictionary-level matching of tool JSON objects, including tool names and arguments. Thus, the reward is not string-level exact matching, but any missing tool call or dictionary-level mismatch yields zero reward for the turn. Dependency-compatible orders within a turn are allowed, since the TDG defines causal constraints rather than a fixed total order.

\paragraph{Local Advantage ($A^{\text{local}}$).}
The local advantage measures the agent's immediate performance relative to the group at turn $t$. We compute the mean $\mu_t^{\text{local}}$ and standard deviation $\sigma_t^{\text{local}}$ of the rewards $\{r_{1,t}, \dots, r_{K,t}\}$ within the group. The standardized local advantage is:
\begin{equation}
A_{i,t}^{\text{local}} = \frac{r_{i,t} - \mu_t^{\text{local}}}{\sigma_t^{\text{local}} + \epsilon}.
\end{equation}
If all rollouts receive the same reward at turn $t$, the relative advantage is set to zero because no within-group preference signal is available; $\epsilon$ is used only for numerical stability. This term incentivizes agents to outperform others at the same turn. Zero normalized advantage is in Appendix~\ref{app:zero_variance_normalization}.

\paragraph{Gated Future Advantage ($A^{\text{future}}$).}
To account for long-horizon reasoning, we assess the future value flow. However, in rigid tool environments, a mistake in the current step often invalidates subsequent tasks. To mitigate this, we introduce a Consistency Gate $\delta_{i,t}$ (typically $\delta_{i,t} = r_{i,t}$), ensuring that future rewards are only credited if the current step is valid. The future value $V_{i,t}$ is defined as the discounted sum of future rewards:
\begin{equation}
V_{i,t} = \delta_{i,t} \cdot \sum_{k=0}^{T-t} \gamma^k r_{i, t+k+1}
\end{equation}
Similar to the local term, we normalize $V_{i,t}$ against the group statistics ($\mu_t^{\text{future}}, \sigma_t^{\text{future}}$) to obtain the relative future advantage $A_{i,t}^{\text{future}}$:
\begin{equation}
A_{i,t}^{\text{future}} = \frac{V_{i,t} - \mu_t^{\text{future}}}{\sigma_t^{\text{future}} + \epsilon}
\end{equation}
This term encourages the agent to select actions that satisfy the current constraint while enabling future success.

\paragraph{Total Advantage Fusion.}
The total advantage used for policy updates is a weighted fusion of the local and future components:
\begin{equation}
A_{i,t}^{\text{total}} = A_{i,t}^{\text{local}} + \lambda \cdot A_{i,t}^{\text{future}}
\end{equation}
where $\lambda \in [0,1)$ is a mixing coefficient, and we use $\lambda=0.5$ by default. By optimizing this decomposed advantage, our method mitigates the credit assignment problem in multi-turn tool interactions, offering denser and more precise supervision than trajectory-level baselines. We also demonstrated the effectiveness of the method in terms of formulas in Appendix \ref{app:theory}.


In summary, we decompose each trajectory into turns, validate each turn with rule-based checks, normalize its outcome against the same-turn group distribution, and assign the resulting advantage to all tokens within that turn.

\section{GUST Dataset}
In this section, we detail the construction of the executable  environments and the automated pipeline employed to synthesize, filter, and curate the \textbf{GUST} dataset.

\subsection{Environments Analysis}
To build a diverse and realistic foundation for agent training, we sourced a vast collection of JSON tool definitions from both open-source and proprietary repositories, covering a wide range of real-world scenarios. As detailed in Section \ref{sec:scaling}, we developed an automated pipeline to convert these static definitions into executable MCP tools. Only toolsets that successfully formed a closed loop—passing all syntax validations and unit tests—were retained. This process resulted in over 400 executable tool environments. To ensure the environments possess appropriate complexity for reasoning training, we strictly filtered for toolsets containing between 5 and 20 tools. Figure \ref{fig:tool_metrics} illustrates the semantic breadth of the environment. We will refactor any toolkit environments that fail to generate tasks successfully in the future.

\begin{figure}[h]
    \centering
    \includegraphics[width=1\linewidth]{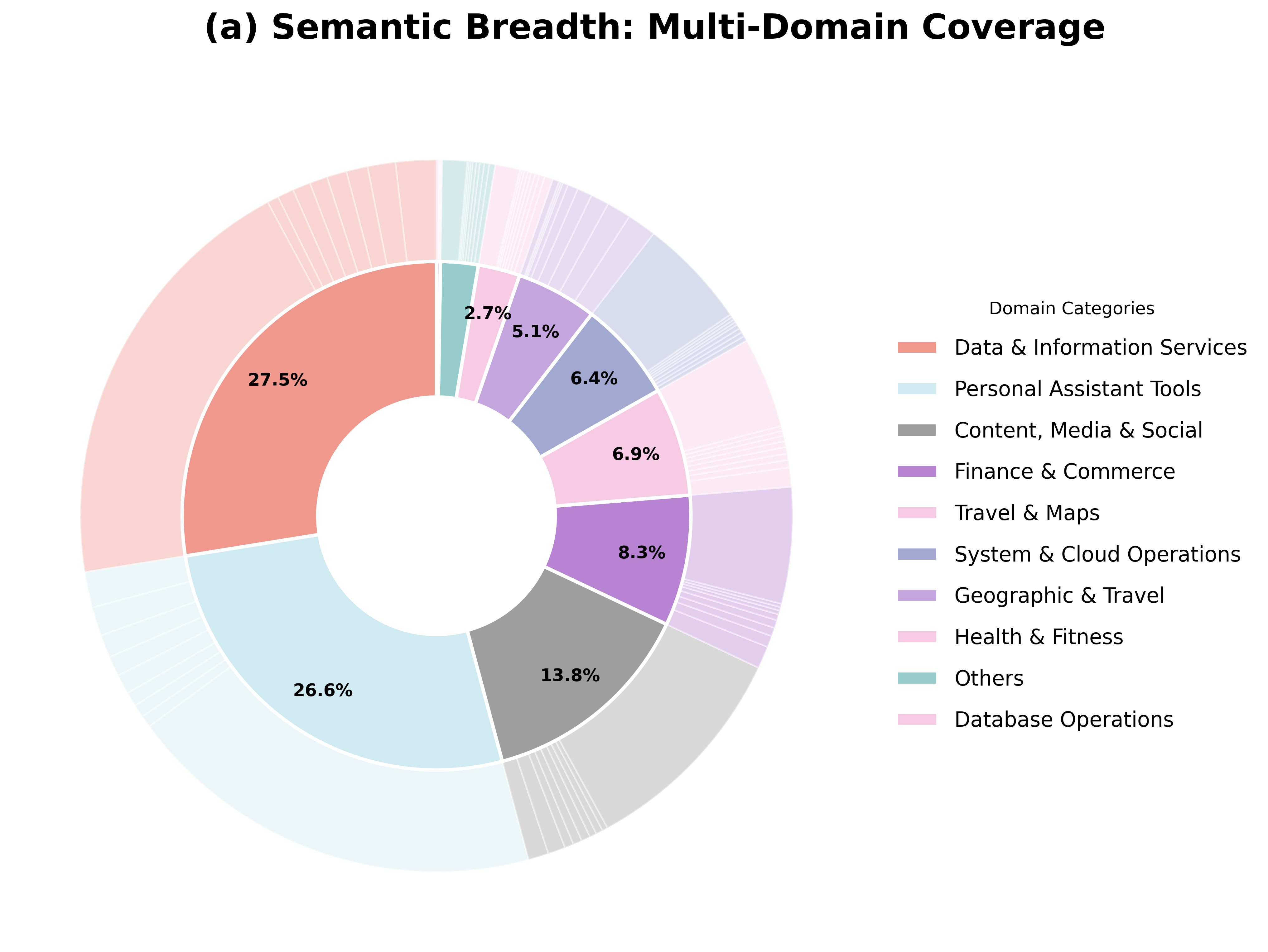}
    \caption{Semantic Breadth: The distribution of tools spans eight macro-domains and numerous specific entities, illustrating the environment's coverage of real-world scenarios.}
    \label{fig:tool_metrics}
    \vskip -0.2in
\end{figure}

\begin{figure}[h]
    \centering
    \includegraphics[width=1\linewidth]{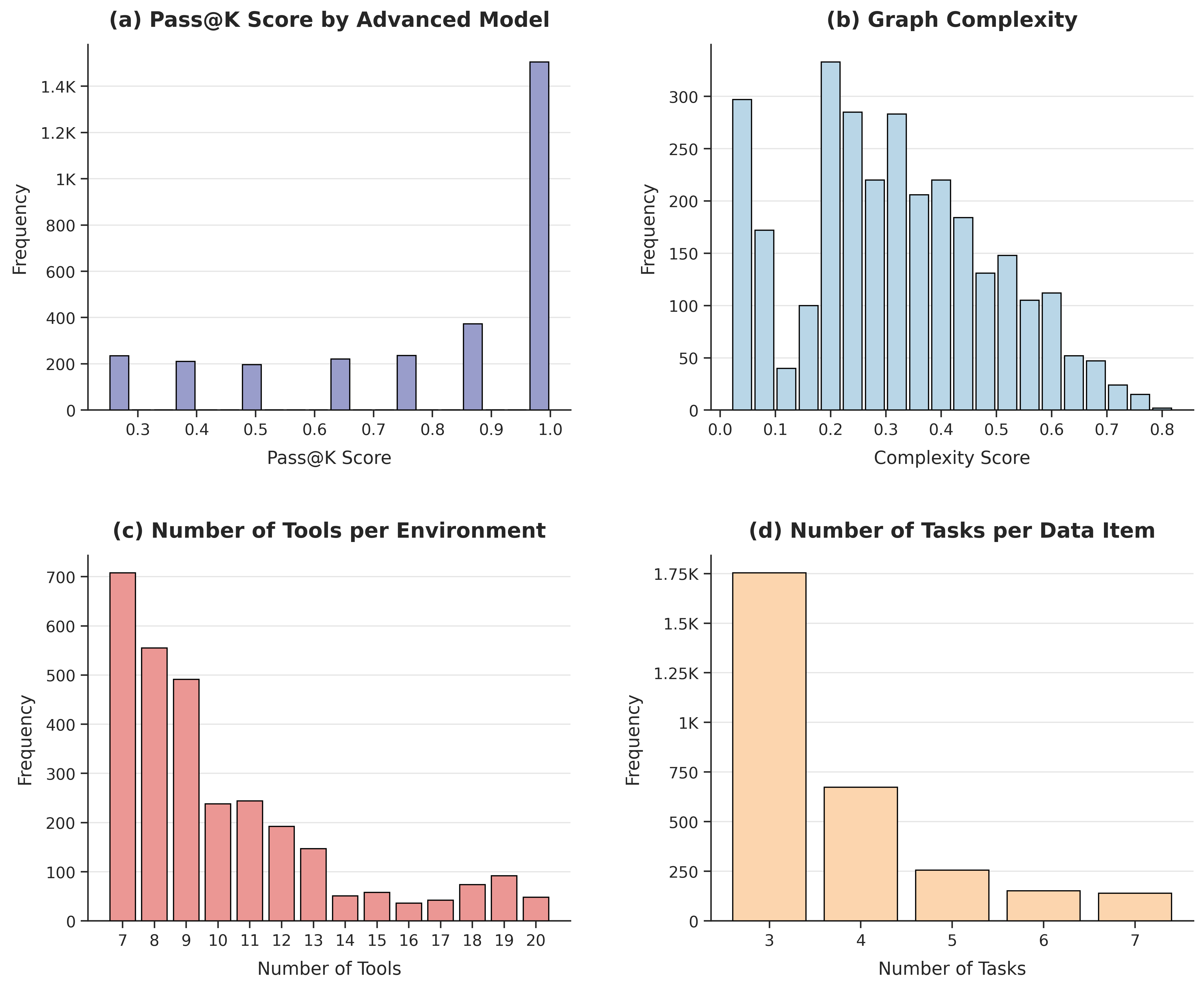}
    \caption{Statistical distribution of the GUST dataset across four metrics: (a) \textbf{Pass@K Score by Advanced Model}, showing a high concentration of successful task completions; (b) \textbf{Graph Complexity}, indicating a diverse range of logic flow difficulty; (c) \textbf{Number of Tools per Environment}, with most environments containing between 7 and 13 tools; and (d) \textbf{Number of Tasks per Data Item}, revealing that the data consist of 3 to 7 tasks.}
    \label{fig:data_metrics}
    \vskip -0.2in
\end{figure}

\subsection{Dataset Analysis} 
Based on the constructed toolsets, we synthesized high-quality training tasks and golden traces using the methods described in Section \ref{sec:graph} and \ref{sec:reward}. For each toolset, we generated three distinct dependency graphs to cover different logical flows. From each graph, we sampled five tasks using our dynamic unlocking algorithm.  To ensure the validity of these synthesized tasks, we implemented a rigorous verification pipeline using \textbf{LangGraph} \cite{LangGraph} as the agent execution framework. We applied a Pass@8 filtering strategy: for each candidate task, a teacher agent (e.g., GPT-4.1 \cite{GPT-4.1}, Deepseek-V3.2 \cite{deepseekai2025deepseekv32pushingfrontieropen}) attempts to solve it up to 8 times (we chose the Qwen3-32B \cite{yang2025qwen3technicalreport} model, which has performance and data distribution similar to the policy model). We calculate the reward for each sampled trajectory using the criteria established in Section \ref{sec:reward}. A task is deemed valid and retained only if at least one trajectory achieves a trace\_score is one; otherwise, the task is discarded. Finally, to maintain dataset diversity and prevent overfitting to specific scenarios, we retained a maximum of ten distinct tasks per toolset environment. Detailed processing results are shown in Figure \ref{fig:data_metrics}. Table \ref{tab:dataset_stats} summarizes the statistics of the final GUST dataset and toolset. Detailed case studies are in Appendix~\ref{fig:case_study}.

\begin{table*}[t]
\caption{Main results on agentic tool-using benchmarks, demonstrating that ToolVerse consistently achieves significant performance gains across all model scales and evaluation environments, with the TARA algorithm achieving the best overall results.}
\label{tab:main_results}
\centering
\small
\setlength{\tabcolsep}{1pt} 
\begin{tabular}{l cccc c cccc ccc}
\toprule
\multirow{2}{*}{\textbf{Model}} & \multicolumn{5}{c}{\textbf{BFCL Multi-Turn}} & \multicolumn{4}{c}{\textbf{$\bm{\tau^2}$-Bench}}  & \multicolumn{3}{c}{\textbf{ACEBench-Agent}} \\
\cmidrule(r){2-6} \cmidrule(lr){7-10} \cmidrule(l){11-13}
& \textbf{Base} & \textbf{\begin{tabular}[c]{@{}c@{}}Miss-\\ Func\end{tabular}} & \textbf{\begin{tabular}[c]{@{}c@{}}Miss-\\ Param\end{tabular}} & \textbf{\begin{tabular}[c]{@{}c@{}}Long-\\ Context\end{tabular}} & \textbf{Overall} & \textbf{Airline} & \textbf{Retail} & \textbf{Telecom}& \textbf{Overall} & \textbf{\begin{tabular}[c]{@{}c@{}}Multi-\\ Step\end{tabular}} & \textbf{\begin{tabular}[c]{@{}c@{}}Multi-\\ Turn\end{tabular}} & \textbf{Overall} \\
\midrule
\textit{\textbf{Advanced Models}} & & & & & & & & & & & & \\
DeepSeek-V3.2 & 41.50 & 39.50 & 33.50 & 35.00 & 37.39 
              & 63.8 & 81.1 & 96.2 & 80.37 
              & 92.50 & 68.75 & 80.62 \\
GPT-4.1       & 47.50 & 32.50 & 32.50 & 43.00 & 38.88 
              & 56.0 & 74.0 & 34.0 & 54.67 
              & 95.00 & 70.83 & 82.92 \\
Kimi-K2-Instruct  & 62.00 & 41.00 & 44.50 & 55.00 & 50.63 
              & 56.5 & 70.6 & 65.8 & 64.30
              & 85.00 & 73.33 & 79.17 \\
\midrule
Qwen3-4B (Thinking) 
        & 32.00 & 20.00 & 24.00 & 25.50 & 25.38 
        & 19.50 & 30.48 & 13.20 & 21.06 
        & 35.91 & 46.66 & 41.28 \\
\rowcolor{rowblue} + ToolVerse (w/ GRPO) 
        & 37.00 & 30.00 & 25.00 & 22.00 & \textbf{28.50}\gain{3.12} 
        & 22.00 & 39.50 & 15.80 & \textbf{25.77}\gain{4.71} 
        & 40.00 & 56.67 & \textbf{48.34}\gain{7.06} \\
\rowcolor{rowblue} + ToolVerse (w/ TARA) 
        & 36.00 & 31.50 & 22.00 & 23.50 & \textbf{28.25}\gain{2.87} 
        & 20.00 & 38.60 & 21.90 & \textbf{26.83}\gain{5.77} 
        & 50.00 & 60.00 & \textbf{55.00}\gain{13.72} \\
\midrule
Qwen3-8B (Thinking) 
        & 32.00 & 33.50 & 22.00 & 28.00 & 28.88
        & 22.00 & 43.20 & 18.40 & 27.87 
        & 53.44 & 39.58 & 46.51 \\
\rowcolor{rowblue} + ToolVerse (w/ GRPO) 
        & 41.50 & 38.00 & 31.00 & 30.50 & \textbf{35.25}\gain{6.37} 
        & 28.00 & 43.00 & 19.30 & \textbf{30.10}\gain{2.23} 
        & 60.00 & 53.33 & \textbf{56.66}\gain{10.15} \\
\rowcolor{rowblue} + ToolVerse (w/ TARA) 
        & 51.50 & 38.50 & 28.50 & 31.50 & \textbf{37.50}\gain{8.62} 
        & 26.00 & 50.00 & 21.10 & \textbf{32.37}\gain{4.50} 
        & 60.00 & 63.33 & \textbf{61.66}\gain{15.15} \\
\midrule
Qwen2.5-14B-Instruct
        & 22.50 & 17.50 & 16.00 & 15.50 & 17.88 
        & 17.00 & 37.73 & 18.42 & 24.38 
        & 40.00 & 55.56 & 47.78 \\
\rowcolor{rowblue} + ToolVerse (w/ GRPO) 
        & 34.50 & 20.00 & 18.50 & 17.00 & \textbf{22.50}\gain{4.62} 
        & 17.50 & 47.60 & 25.40 & \textbf{30.17}\gain{5.79} 
        & 45.00 & 56.67 & \textbf{50.84}\gain{3.06} \\
\rowcolor{rowblue} + ToolVerse (w/ TARA) 
        & 40.50 & 16.50 & 18.00 & 21.50 & \textbf{24.12}\gain{6.24} 
        & 20.00  & 45.60 & 31.60 & \textbf{32.40}\gain{8.02} 
        & 60.00 & 63.33 & \textbf{61.66}\gain{13.88} \\
\bottomrule
\end{tabular}
\end{table*}

\section{Experiments}
\label{sec:results}
\subsection{Experiment Setup}
\textbf{Training.} We perform RL on the Qwen2.5-14B-Instruct \cite{qwen2025qwen25technicalreport}, Qwen3-4B, and Qwen3-8B models \cite{yang2025qwen3technicalreport} using Group Relative Policy Optimization (GRPO) \cite{grpo}. A high-fidelity user simulator, DeepSeek-V3.2 \cite{deepseekai2025deepseekv32pushingfrontieropen}, is deployed to interact dynamically with the agent. We train using the custom-built GUST dataset and environment, where per environment acts as a MCP toolset. Modifications to the Verl framework \cite{sheng2024hybridflow} are made to seamlessly integrate our MCP tool environment during training. Our experiments are conducted on a cluster of 32 NVIDIA A100 GPUs. The training configuration employs a global batch size of 128. Detailed hyperparameters are documented in Appendix \ref{details_experiment}.


\textbf{Evaluation.} We evaluate the effectiveness of \textsc{ToolVerse} in enhancing long-horizon reasoning and tool-use capabilities across three widely-adopted multi-turn benchmarks: (1)BFCL-v3 Multi-Turn~\cite{bfcl}, which involves Python-based API interactions across four specialized categories, including \textit{Base}, \textit{Miss Param}, \textit{Miss Func}, and \textit{Long Context}; (2)$\tau^2$-Bench~\cite{tau2-bench}, which assesses complex user-agent interactions in real-world domains such as Airline, Retail, and Telecom; and (3)ACEBench-Agent~\cite{ACEBench}, which focuses on multi-step and multi-turn reasoning in dynamic environments. For $\tau^2$-Bench and ACEBench-Agent, which require a conversational user, we employ \texttt{DeepSeek-V3.2}\cite{deepseekai2025deepseekv32pushingfrontieropen} as the user simulator. Notably, since the original ACEBench uses a \texttt{[func\_name(param)]} prompt format by default, we modify the official implementation to support the LLMs’ native function-calling interface to ensure consistency. To ensure the reliability and stability of our experimental results, we conduct four independent runs for each evaluation and report the average@4 scores.

\begin{table}[t]
\centering
\small
\setlength{\tabcolsep}{3pt}
\begin{tabular}{lcc}
\toprule
\textbf{Method} & \textbf{BFCL-v3 Acc.} & \textbf{$\bm{\tau^2}$-Bench Avg.} \\
\midrule
Qwen2.5-7B-Instruct & 12.88\% & 16.00\% \\
+ToolRL             & 15.25\% & 16.37\% \\
+agentflow          & 11.25\% & 17.03\% \\
+SimpleTIR          & 11.75\% & 13.53\% \\
\textbf{+TARA (Ours)} & \textbf{20.00\%} & \textbf{29.87\%} \\
\bottomrule
\end{tabular}
\caption{Comparison with publicly available and reproducible baselines on BFCL Multi-turn and $\bm{\tau^2}$-Bench.}
\label{tab:public_baselines}
\end{table}

\subsection{Main Results}

\textbf{ToolVerse consistently delivers universal performance gains across diverse benchmarks and model architectures.} Experimental results in Table \ref{tab:main_results} demonstrate that ToolVerse consistently yields significant improvements across all three benchmarks and all evaluated model scales. Notably, this framework proves highly effective for both "Thinking" models and "Non-thinking" models; for instance, Qwen3-8B (Thinking) achieves a remarkable gain of +15.15 on ACEBench-Agent, while the non-thinking Qwen2.5-14B-Instruct also sees a substantial +13.88 increase on the same benchmark. This universal success across various testing dimensions—ranging from missing parameter detection in BFCL to complex domain-specific reasoning in $\tau^2$-Bench—validates that ToolVerse provides a robust and model-agnostic training paradigm for enhancing agentic tool-use capabilities.


\textbf{Our Turn-Aware Relative Advantage algorithm provides substantial gains over the naive GRPO baseline.} To isolate the impact of our proposed reward decomposition, we conduct an ablation study by comparing ToolVerse trained with naive GRPO (w/ GRPO) against our full TARA algorithm (w/ TARA) in Table \ref{tab:main_results}. The results consistently show that TARA yields superior performance across all benchmarks. For instance, on ACEBench-Agent, while naive GRPO improves the Qwen3-8B model to 56.66\%, TARA further elevates it to 61.66\%, representing a significant additional boost. This performance gap is even more pronounced in the multi-turn and multi-step subtasks, where the credit assignment problem is most acute. By replacing the coarse trajectory-level advantage in GRPO with fine-grained, turn-level reinforcement signals, TARA effectively identifies and reinforces correct intermediate tool-calling logic, leading to a more robust and precise policy in complex, long-horizon environments.

We evaluate against the publicly available baselines we could identify and evaluate in a reproducible manner, including ToolRL, agentflow, and SimpleTIR. As shown in Table~\ref{tab:public_baselines}, Qwen2.5-7B-Instruct-TARA consistently outperforms these public baselines. For \textbf{SALT} and \textbf{FTRL}, a faithful head-to-head comparison is currently not feasible because the key experimental artifacts are not publicly available. We will clarify this limitation in the revision.

\begin{table}[h]
\centering
\small
\setlength{\tabcolsep}{4pt}
\begin{tabular}{lcc}
\toprule
\textbf{Method} & \textbf{BFCL-v3} & \textbf{$\bm{\tau^2}$-Bench} \\
\midrule
Qwen3-8B & 28.88\% & 27.87\% \\
Qwen3-8B + GRPO & 35.25\% & 30.10\% \\
Turn-local only & 33.75\% & 27.33\% \\
Turn-local + future w/o gate & 35.00\% & 28.17\% \\
Full TARA & \textbf{37.50\%} & \textbf{32.37\%} \\
\bottomrule
\end{tabular}
\caption{Ablation of TARA on BFCL-v3 and $\bm{\tau^2}$-Bench.}
\label{tab:tara_component_ablation}
\end{table}

\subsection{Ablation Study}
\label{sec:ablation}

We conduct ablation studies on Qwen3-8B to evaluate the contribution of TARA components, the sensitivity to $\lambda$ and $\gamma$, and the effect of environment scaling, with detailed results reported in Appendix~\ref{app:other}.

\paragraph{Component ablation.}
Table~\ref{tab:tara_component_ablation} shows that full TARA achieves the best performance on both BFCL-v3 and $\tau^2$-Bench. The turn-local variant improves over the base model but remains weaker than GRPO, indicating that local credit alone is insufficient for long-horizon tool-use optimization. Adding future credit without the consistency gate also underperforms full TARA, suggesting that the gate is important for filtering noisy future signals.

\paragraph{Sensitivity to $\lambda$ and $\gamma$.}
As shown in Figure~\ref{fig:tara_hyper_sensitivity}, TARA is reasonably robust to hyperparameter choices and performs best at $\lambda=0.5$ and $\gamma=0.5$. Performance degrades when either value becomes too small or too large, indicating that moderate future-credit propagation provides a better balance between long-horizon reward assignment and optimization stability.

\paragraph{Effect of environment scaling.}
Table~\ref{tab:env_scaling_ablation} shows that increasing the number of ToolVerse environments generally improves generalization. Scaling from 100 to the full 422 environments raises BFCL-v3 from 35.00\% to 37.50\% and $\tau^2$-Bench from 27.33\% to 32.37\%, highlighting the importance of environment diversity for robust agentic tool use.

\subsection{Training Dynamics}
\label{subsec:TARA_algorithm_analysis}


The training curves in Figure \ref{fig:train_curves} further substantiate the quantitative results. In the Trace Score plot, the Qwen3-8B model with TARA shows a clear and continuous improvement over training steps, with a sharper increase compared to the Naive GRPO approach. The Val Score plot similarly indicates stronger convergence for the TARA-enhanced model, reaching higher validation scores in fewer training steps. This observation suggests that TARA not only accelerates the learning process but also improves the stability of the learning trajectory, leading to faster convergence and better overall performance during the training process.

\section{Conclusion}
In this work, we presented \textbf{ToolVerse}, a framework that advances agentic reinforcement learning by scaling executable environments and synthesizing long-horizon tasks via the graph-based dynamic unlocking sampling algorithm. We have organized these tools and data into the GUST dataset. By integrating these diverse scenarios with a novel turn-aware credit assignment policy, our approach effectively addresses the challenges of sparse rewards and multi-step reasoning. We proved the correctness and effectiveness of the turn-aware advantage by comparing the training curves with the form of the formula derivation. Empirical results demonstrate that ToolVerse achieves advanced performance in complex tool-use benchmarks. In the future, we will continue to expand more advanced task generation methods in both simulated and real-world environments, and we will also continue to explore more scalable and fine-grained credit assignment strategies.

\section*{Limitations}
Despite the significant advancements enabled by ToolVerse in scaling agentic RL environments and improving long-horizon tool-integrated reasoning, our work has several limitations that warrant further investigation. 
First, although our dynamic task generation strategy leverages tool dependency graphs for coherent multi-step reasoning tasks, it relies on predefined dependencies within available protocols. This could limit the diversity of emergent behaviors compared to environments where novel tool interactions can be discovered autonomously. 
Second, our Turn-Aware Relative Advantage algorithm improves credit assignment granularity but still depends on reward signals defined at each turn; in highly sparse or ambiguous reward settings, the approach may struggle to provide sufficient feedback for optimal policy learning.



\bibliography{refer_papers}

\clearpage
\appendix

\section{Theoretical Analysis of Turn-Aware Relative Advantage}
\label{app:theory}

In this section, we provide a theoretical justification for the Turn-Aware Relative Advantage mechanism proposed in Section \ref{sec:reward}. We focus on two key properties: (1) variance reduction in the advantage estimation, and (2) the suppression of false positive credits (distractors) via the consistency gate.

\subsection{Preliminaries and Definitions}

Let $\tau = (s_1, a_1, \dots, s_T, a_T)$ denote a trajectory of length $T$. In the standard Group Relative Policy Optimization (GRPO), the advantage for an action $a_t$ is typically estimated using the cumulative return of the entire trajectory or the return-to-go. Let $G_t^{\text{std}}$ denote the standard return formulation used for advantage estimation at step $t$:
\begin{equation}
    G_t^{\text{std}} = r(s_t, a_t) + \sum_{k=1}^{T-t} \gamma^k r(s_{t+k}, a_{t+k})
\end{equation}
where the first term is the immediate reward and the summation represents the future return $R_{t+1:T}$.

In our proposed method, the advantage is derived from a fused value estimate $G_t^{\text{ours}}$. Based on the definitions of Local Advantage and Gated Future Advantage, the underlying value estimate can be expressed as:
\begin{equation}
    G_t^{\text{ours}} = r_{i,t} + \lambda \cdot V_{i,t} = r_{i,t} + \lambda \cdot \delta_{i,t} \cdot \sum_{k=1}^{T-t} \gamma^k r_{i, t+k}
\end{equation}
Here, $r_{i,t}$ is the deterministic binary reward based on the golden trace, $\lambda \in [0.5, 1.0]$ is the fusion weight, and $\delta_{i,t} \in \{0, 1\}$ is the consistency gate.

\subsection{Variance Reduction}

A primary challenge in long-horizon reasoning is the high variance of future returns, which introduces noise into the gradient estimation for earlier steps.

\begin{assumption}
    We assume that the immediate reward $r_{i,t}$ is deterministic given the state $s_t$ and action $a_i,t$ (i.e., strict matching with the golden trace). The future return $R_{t+1:T} = \sum_{k=1}^{T-t} \gamma^k r_{i, t+k}$ is a random variable with variance $\sigma^2_{\text{future}}$ due to the stochasticity of the policy and environment in subsequent steps.
\end{assumption}

\begin{theorem}[Variance Reduction]
    \label{thm:variance}
    For any valid reasoning step where $r_{i,t}=1$ (and thus $\delta_{i,t}=1$), the variance of the Turn-Aware value estimate is strictly lower than that of the standard return-to-go estimate, provided that the fusion hyperparameter satisfies $\lambda < 1$.
\end{theorem}

\begin{proof}
    Consider the variance of the standard estimator $G_t^{\text{std}}$:
    \begin{align}
        \text{Var}(G_t^{\text{std}}) &= \text{Var}(r_{i,t} + R_{t+1:T}) \\
        &= \text{Var}(r_{i,t}) + \text{Var}(R_{t+1:T}) + 2\text{Cov}(r_{i,t}, R_{t+1:T})
    \end{align}
    Since $r_{i,t}$ is deterministic given the current step match, $\text{Var}(r_{i,t}) = 0$ and the covariance term vanishes. Thus:
    \begin{equation}
        \text{Var}(G_t^{\text{std}}) = \text{Var}(R_{t+1:T}) = \sigma^2_{\text{future}}
    \end{equation}
    
    Now consider our proposed estimator $G_t^{\text{ours}}$ given a valid current step ($\delta_{i,t}=1$):
    \begin{align}
        \text{Var}(G_t^{\text{ours}}) &= \text{Var}(r_{i,t} + \lambda \cdot 1 \cdot R_{t+1:T}) \\
        &= \lambda^2 \text{Var}(R_{t+1:T}) \\
        &= \lambda^2 \sigma^2_{\text{future}}
    \end{align}
    Since $\lambda < 1$, it follows that $\lambda^2 < 1$, and therefore:
    \begin{equation}
        \text{Var}(G_t^{\text{ours}}) < \text{Var}(G_t^{\text{std}})
    \end{equation}
    Furthermore, for an invalid step ($r_{i,t}=0$), the gate $\delta_{i,t}=0$ forces $V_{i,t}=0$, resulting in $\text{Var}(G_t^{\text{ours}}) = 0$. In both cases, our method significantly reduces the variance introduced by downstream uncertainty.
\end{proof}

\subsection{Consistency Gating and Distractor Suppression}

In tool-use environments, an agent may select an incorrect action (a ``distractor'') but still achieve a positive outcome later due to spurious correlations or environment tolerance. Standard RL reinforces such distractors.

\begin{definition}[Distractor Action]
    An action $a_{i,t}$ is defined as a distractor if it deviates from the golden trace (i.e., $r_{i,t} = 0$), yet the subsequent trajectory yields a positive future return $R_{t+1:T} > 0$.
\end{definition}

\begin{theorem}[Distractor Suppression]
    \label{thm:distractor}
    Under the Turn-Aware Relative Advantage mechanism, the total advantage $A_{i,t}^{\text{total}}$ for any distractor action is guaranteed to be non-positive, regardless of the magnitude of future returns.
\end{theorem}

\begin{proof}
    Let $a_{i,t}$ be a distractor action. By definition, $r_{i,t} = 0$.
    
    First, we analyze the \textbf{Local Advantage} $A_{i,t}^{\text{local}}$. The standardized advantage is calculated as:
    \begin{equation}
        A_{i,t}^{\text{local}} = \frac{r_{i,t} - \mu_t^{\text{local}}}{\sigma_t^{\text{local}} + \epsilon} = \frac{0 - \mu_t^{\text{local}}}{\sigma_t^{\text{local}} + \epsilon}
    \end{equation}
    Assuming the group size $K$ is sufficiently large and the task is solvable, at least one sampled trajectory (or the theoretical optimal) will match the trace, implying the group mean reward $\mu_t^{\text{local}} > 0$. Consequently, $A_{i,t}^{\text{local}} < 0$.
    
    Second, we analyze the \textbf{Future Advantage} $A_{i,t}^{\text{future}}$. The consistency gate is defined as $\delta_{i,t} = r_{i,t}$. Since $r_{i,t} = 0$, the gate closes:
    \begin{equation}
        V_{i,t} = \delta_{i,t} \cdot \sum_{k=0}^{\infty} \gamma^k r_{i, t+k+1} = 0 \cdot R_{t+1:T} = 0
    \end{equation}
    Even if the future return $R_{t+1:T}$ is high (lucky guess), it is blocked. The relative future advantage becomes:
    \begin{equation}
        A_{i,t}^{\text{future}} = \frac{0 - \mu_t^{\text{future}}}{\sigma_t^{\text{future}} + \epsilon}
    \end{equation}
    If the group generally performs well ($\mu_t^{\text{future}} > 0$), this term is negative. If the group performs poorly ($\mu_t^{\text{future}} \approx 0$), this term is zero.
    
    Finally, the \textbf{Total Advantage} is:
    \begin{equation}
        A_{i,t}^{\text{total}} = A_{i,t}^{\text{local}} + \lambda A_{i,t}^{\text{future}}
    \end{equation}
    Since $A_{i,t}^{\text{local}}$ is strictly negative and $\lambda A_{i,t}^{\text{future}}$ is non-positive, the total advantage $A_{i,t}^{\text{total}}$ is strictly negative. Thus, the probability of the distractor action $\pi(a_{i,t}|s_t)$ will be suppressed, eliminating the credit assignment error found in baseline methods.
\end{proof}

\section{Implementation Details}
\label{details_experiment}

\subsection{Environment Construction Details.}
\label{app:environment_construction}

We construct executable tool environments from raw tool schemas extracted from tool lists in system prompts of open-source tool-agentic data. In particular, we use the multi-turn subset of the Toucan dataset~\citep{xu2025toucansynthesizing15mtoolagentic}, where each example provides detailed tool specifications, including function names, natural-language descriptions, input arguments, and JSON schemas. Starting from these raw tool definitions, we first normalize the schema format by removing irrelevant textual artifacts, standardizing argument names and types, and converting each tool list into a consistent MCP-compatible specification.

For each toolset, we build a lightweight mock database implemented as Python dictionaries. The database construction is scenario-driven: given the JSON tool list, we first infer the underlying application scenario, possible business entities, and state variables required by the tools, such as users, orders, tickets, products, reservations, or inventory records. We then synthesize 3--5 coherent records for each scenario to support stateful interactions. These mock records are designed to be internally consistent across tools, so that outputs from query-style tools can serve as valid inputs to subsequent update or action-style tools.

Given the normalized schemas and the corresponding mock database, we generate executable MCP tool functions in Python. Each generated function either queries or updates the dictionary database according to the semantics of the original tool definition. We adopt an iterative generation-and-verification procedure: after code generation, we run syntax checks and unit tests to verify both executability and functional correctness. The unit tests cover whether each function can be called with valid arguments, whether it returns outputs conforming to the expected schema, and whether state-changing tools correctly update the mock database. If a toolset fails validation, we regenerate or repair the implementation and test it again. Toolsets that repeatedly fail to yield a reliable executable mock environment are discarded. In practice, approximately 20\% of the initial toolsets are filtered out, mostly because their APIs depend on external services, complex real-time states, or domain-specific side effects that are difficult to faithfully mock with a dictionary-based backend. The remaining toolsets form the executable MCP environments used for task synthesis and reinforcement learning.

\subsection{TDG Sanitization and Dependency Noise.}
\label{app:tdg_sanitization}

In our task synthesis pipeline, each node in the tool dependency graph (TDG) corresponds to a structured tool definition, including the tool name, natural-language description, input schema, and output schema. Directed edges are inferred by an LLM judge according to two types of dependencies: parameter dependency, where the output of one tool can provide a required argument for another tool, and semantic dependency, where one tool should logically precede another within the same application scenario. Although the LLM-based judge may introduce noisy or redundant edges, such noise does not directly lead to inconsistent supervision, because the task instruction, golden trace, and rule-based validator are all constructed with respect to the same executable dependency structure. In other words, the generated task and its ground-truth trajectory remain aligned under the sanitized TDG. Moreover, trajectories that cannot be instantiated or executed successfully are removed by the subsequent execution-based verification stage.

Before applying Dynamic Unlocking Sampling, we further sanitize each LLM-induced TDG to ensure that it is a directed acyclic graph. Starting from root nodes with zero in-degree, we perform a topological traversal and maintain only dependency edges that are reachable under the current partial order. If a cycle is detected, we iteratively remove the lowest-priority edge in the cycle, where priority is determined by the dependency type and the LLM confidence score when available. Parameter-dependency edges are preserved with higher priority than semantic-dependency edges, since they are directly grounded in tool schemas and executable arguments. If no confidence score is available, we remove semantic edges first and break ties by removing the edge whose removal maximizes the number of newly unlocked nodes. This procedure guarantees that the sanitized graph is acyclic before DUS is applied. The resulting DAG preserves the major dependency structure while enabling a valid topological unlocking process for long-horizon task construction.

\subsection{Inverse Context Reconstruction.}
\label{app:inverse_context_reconstruction}

DUS produces an abstract dependency-compatible tool skeleton, represented as a sequence of tool JSON objects without fully grounded runtime arguments. We convert this skeleton into an executable tool-call sequence through inverse context reconstruction. Given a sampled sequence, we instantiate each tool call in topological order under the current mock database state. Arguments are filled from two sources. First, for dependency-based arguments, we follow the incoming TDG edges and copy values from the outputs of previously executed tools. Second, for context-specific arguments that are not determined by previous outputs, we retrieve valid values from the mock database, such as user IDs, order IDs, product names, reservation records, or ticket identifiers. Whenever a tool changes the database state, the updated state is immediately synchronized and used when instantiating later tool calls. The resulting executable sequence is treated as the \textbf{Golden Trace}.

For each environment, we do not generate only one trace. Instead, we repeatedly sample dependency-compatible tool skeletons and instantiate them until the resulting trace collection covers the tools in the environment as broadly as possible. This coverage-oriented sampling prevents the generated data from collapsing onto a small subset of easy or frequently unlocked tools. The final collection of traces for an environment is then converted into a task set.

To generate natural-language user tasks, we prompt an LLM with each Golden Trace together with the tool definitions and the incoming dependency edges of the involved TDG nodes. The dependency edges are included explicitly so that the generated task reflects the intended causal structure of the trace. For example, when a later tool requires an identifier, status, or intermediate value produced by an earlier tool, the generated user instruction is encouraged to require this dependency rather than allowing independent tool calls. This produces tasks whose solutions naturally require multi-step tool use.

After task generation, we perform execution-based verification using LangGraph. The generated tasks are replayed in the executable MCP environment, and the corresponding Golden Trace is executed against the mock database. This step is necessary because the environment is stateful: earlier tool calls may update the database, which can change the valid arguments, return values, or preconditions of subsequent tools. A task set is retained only if its Golden Trace can be executed successfully under these state transitions. Finally, we apply teacher-agent Pass@8 filtering, where a strong teacher agent attempts each task multiple times. Tasks that cannot be solved by the teacher agent in any of the attempts are discarded, yielding the final verified task set used for RL training.

For turn-level normalization, we use $\epsilon=10^{-6}$. When the reward or future-value variance of a group at a given turn is zero, the corresponding normalized advantage is set to zero, since no relative preference signal is available within the group.

\subsection{Zero-Variance Handling in Turn-Level Normalization}
\label{app:zero_variance_normalization}

TARA computes turn-level relative advantages by normalizing each rollout's turn-level signal against the group distribution at the same turn. This normalization is applied to both the local reward signal and the gated future-value signal. In general, for a scalar turn-level signal $x_{i,t}$ of the $i$-th rollout at turn $t$, where $x_{i,t}$ can be either the local binary reward $r_{i,t}$ or the gated future value $V_{i,t}$, we compute:
\begin{equation}
\mu_t^x = \frac{1}{K}\sum_{i=1}^{K} x_{i,t},
\qquad
\sigma_t^x = \sqrt{\frac{1}{K}\sum_{i=1}^{K}(x_{i,t}-\mu_t^x)^2}.
\end{equation}
The corresponding normalized advantage is:
\begin{equation}
A_{i,t}^{x} =
\frac{x_{i,t}-\mu_t^x}{\sigma_t^x+\epsilon},
\end{equation}
where $\epsilon=10^{-6}$ is used for numerical stability.

A degenerate case occurs when all rollouts in the group receive the same value at turn $t$, i.e.,
\begin{equation}
x_{1,t}=x_{2,t}=\cdots=x_{K,t}.
\end{equation}
In this case, we have:
\begin{equation}
\mu_t^x = x_{i,t}, \qquad \sigma_t^x = 0,
\end{equation}
for every rollout $i$. Therefore, the numerator of the normalized advantage is also zero:
\begin{equation}
x_{i,t}-\mu_t^x = 0.
\end{equation}
The normalized advantage thus becomes:
\begin{equation}
A_{i,t}^{x} =
\frac{0}{0+\epsilon}=0.
\end{equation}

Intuitively, this case indicates that the group provides no relative preference signal at this turn. For example, all rollouts may have correctly covered the golden tool calls, or all rollouts may have failed in the same way. Since GRPO-style optimization relies on relative comparisons within a sampled group, such a turn should not push the policy toward or away from any particular rollout. We therefore set the corresponding normalized advantage to zero for both local and future components:
\begin{equation}
A_{i,t}^{\mathrm{local}}=0
\quad \text{if} \quad
\sigma_t^{\mathrm{local}}=0,
\end{equation}
and
\begin{equation}
A_{i,t}^{\mathrm{future}}=0
\quad \text{if} \quad
\sigma_t^{\mathrm{future}}=0.
\end{equation}
Consequently, the total turn-level advantage only receives contributions from components whose group-level variance is non-zero. This prevents degenerate turns from introducing artificial gradients while preserving valid relative learning signals from other turns.

\subsection{Experimental Details}
To ensure the reproducibility of our experiments, we provide the specific hyperparameter configurations used during the training phase. The model is optimized using a learning rate of $1 \times 10^{-6}$ with a mini-batch size of $32$. To manage sequence complexity, we set the maximum prompt length to $8,192$ tokens and the total response length to $24,000$ tokens, while limiting each individual turn's response to $1,024$ tokens. Data preprocessing includes shuffling and the filtering of overlong prompts to maintain training stability. Notably, the KL divergence loss is disabled in this setup ($kl\_loss=\text{False}$). For our proposed advantage mixing mechanism, we employ a mixing weight $\lambda_{mix} = 0.5$ for future advantages and a discount factor $\gamma = 0.5$ to balance immediate and long-term rewards.

\begin{table}[ht]
    \centering
    \caption{Summary of Experimental Hyperparameters}
    \label{tab:hyperparameters}
    \begin{small}
        \begin{tabular}{lc}
            \toprule
            \textbf{Hyperparameter} & \textbf{Value} \\
            \midrule
            Learning Rate & $1 \times 10^{-6}$ \\
            Mini-batch Size & $32$ \\
            Max Prompt Length & $8,192$ \\
            Max Response Length & $24,000$ \\
            Single-turn Response Length & $1,024$ \\
            KL Loss ($kl\_loss$) & False \\
            Mixing Weight ($\lambda_{mix}$) & $0.5$ \\
            Discount Factor ($\gamma$) & $0.5$ \\
            Data Shuffling & True \\
            Filter Overlong Prompts & True \\
            \bottomrule
        \end{tabular}
    \end{small}
\end{table}

\section{Related Experiments and Results}
\label{app:other}

\begin{table}[h]
    \caption{Statistics of the constructed GUST dataset. }
    \label{tab:dataset_stats}
    \centering
    \begin{small}
        \begin{tabular}{l|c}
        \toprule
        \textbf{Statistic} & \textbf{Count} \\
        \midrule
        Total Environments & 422 \\
        Total Executable Tools & 4438 \\
        Total Data Items & 2987 \\
        Avg. Tasks per data item & 3.8 \\
        Avg. Tools per environment & 12.2 \\
        \bottomrule
        \end{tabular}
    \end{small}
\end{table}

\begin{table}[h]
\centering
\small
\setlength{\tabcolsep}{4pt}
\begin{tabular}{lcc}
\toprule
\textbf{Training Environments} & \textbf{BFCL-v3} & \textbf{$\bm{\tau^2}$-Bench} \\
\midrule
Qwen3-8B & 28.88\% & 27.87\% \\
ToolVerse-100 & 35.00\% & 27.33\% \\
ToolVerse-200 & 35.75\% & 27.57\% \\
ToolVerse-300 & 36.00\% & 32.12\% \\
ToolVerse-Full(422) & \textbf{37.50\%} & \textbf{32.37\%} \\
\bottomrule
\end{tabular}
\caption{Effect of environment scaling. Table 7 evaluates how the number of ToolVerse training environments affects cross-benchmark generalization. Increasing the environment scale generally improves performance on external tool-use benchmarks. In particular, scaling from 100 to the full 422 environments improves BFCL-v3 from 35.00\% to 37.50\% and $\bm{\tau^2}$-Bench from 27.33\% to 32.37\%. These results suggest that environment diversity, rather than merely additional rollouts from a fixed environment set, is important for learning robust agentic tool-use behaviors.}
\label{tab:env_scaling_ablation}
\end{table}

\begin{figure}[h]
\centering
\includegraphics[width=\linewidth]{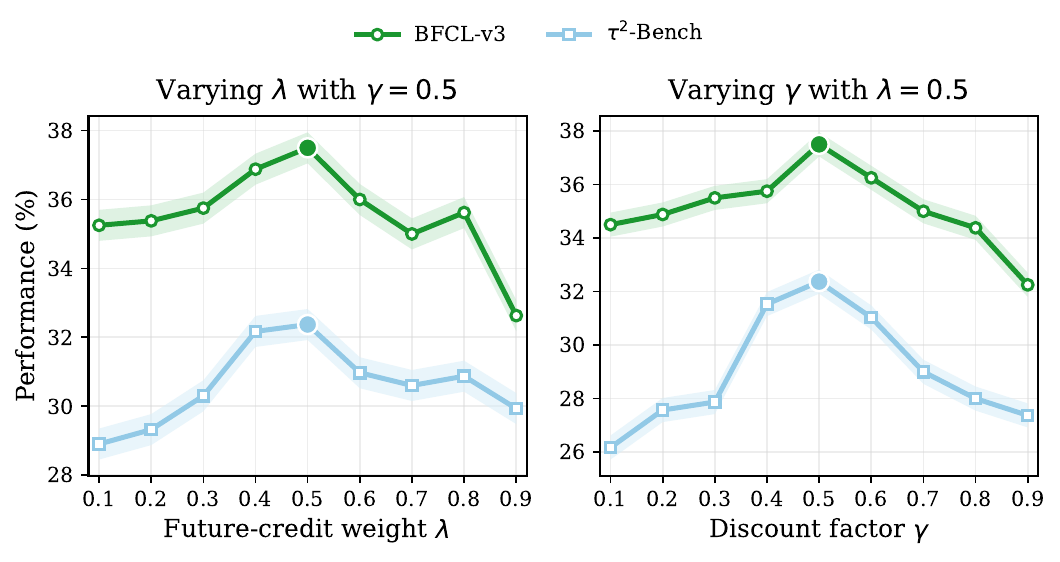}
\caption{Sensitivity analysis of TARA with $\lambda$ and $\gamma$.}
\label{fig:tara_hyper_sensitivity}
\end{figure}

\begin{figure}[h]
    \centering  
    \includegraphics[width=1\linewidth]{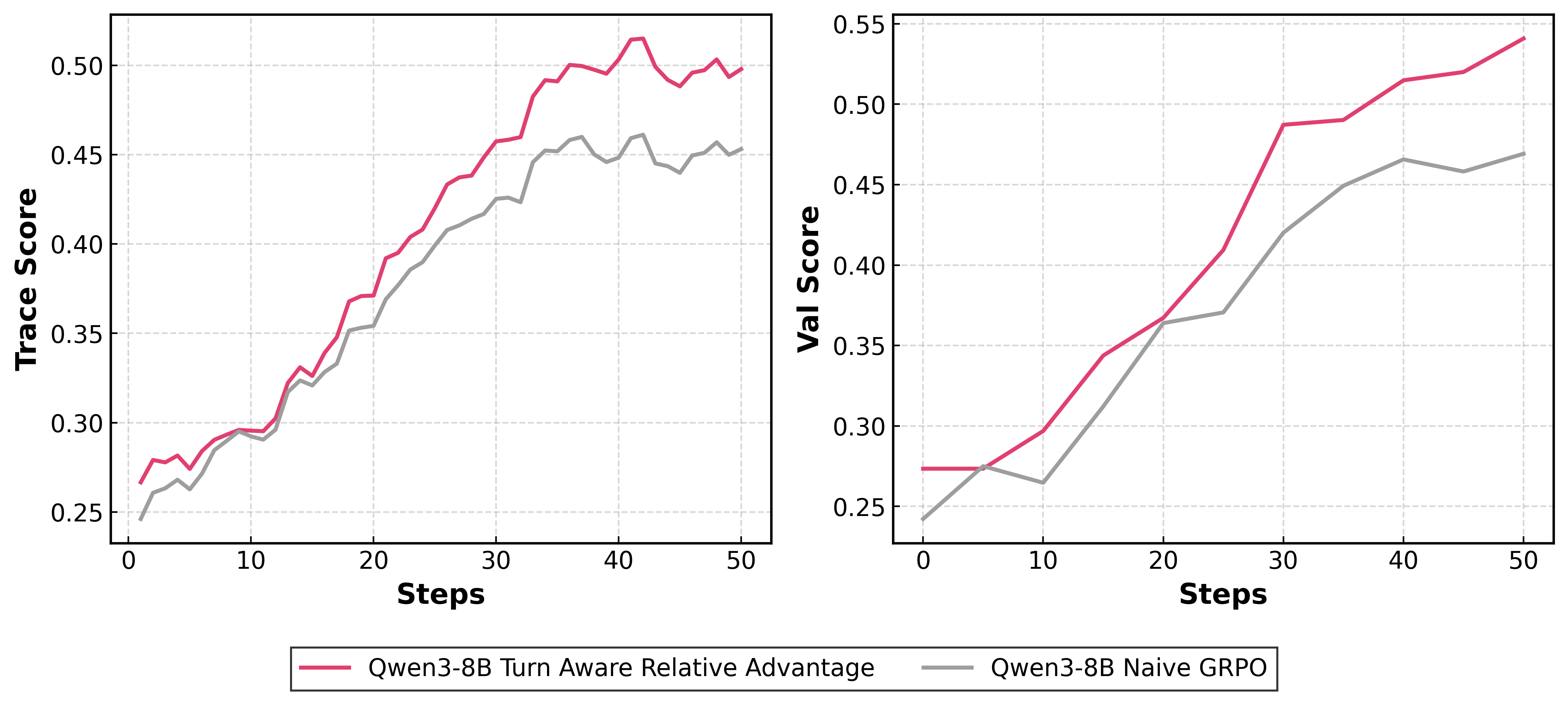}
    \caption{\textbf{Training Dynamics Comparison: TARA vs. Naive GRPO.} We report the Trace Score (left) and Validation Score (right) for Qwen3-8B throughout the training process. The TARA-enhanced model (pink) outperforms the Naive GRPO model (gray), with a more consistent and faster improvement in both scores across training steps. The Trace Score evaluates binary task completion, reflecting whether the agent successfully fulfills the final goal by executing every step in the trajectory correctly.}
    \label{fig:train_curves}
    \vskip -0.2in
\end{figure}

\subsection{Ablation Study}

We conduct ablation studies to analyze the contribution of the main components in TARA, the sensitivity to key hyperparameters, and the effect of scaling the number of training environments. Unless otherwise specified, all ablations are performed with Qwen3-8B under the same training recipe as the main experiments. \paragraph{Component ablation.} Table~\ref{tab:tara_component_ablation} studies the contribution of each component in TARA. The turn-local variant uses only the turn-level advantage, corresponding to $\lambda=0$ and thus excluding future credit. This variant improves over the base model on BFCL-v3, but it is weaker than GRPO and does not improve the average performance on $\tau^2$-Bench. This suggests that purely local credit assignment is insufficient for long-horizon tool-use tasks. Adding future credit without the consistency gate improves BFCL-v3 to 35.00\%, but still underperforms full TARA, especially on $\tau^2$-Bench. Full TARA achieves the best results on both benchmarks, reaching 37.50\% on BFCL-v3 and 32.37\% on $\tau^2$-Bench. These results indicate that both future-aware credit propagation and the consistency gate are important for stable and effective optimization. \paragraph{Sensitivity to $\lambda$ and $\gamma$.} We further examine the sensitivity of TARA to the future-credit weight $\lambda$ and the discount factor $\gamma$. As shown in Figure~\ref{fig:tara_hyper_sensitivity}, TARA is not dependent on a single narrowly tuned configuration. When fixing $\gamma=0.5$, performance generally improves as $\lambda$ increases from 0.1 to 0.5, with the best result obtained at $\lambda=0.5$. However, assigning too much weight to future credit leads to degraded performance, indicating that overemphasizing delayed rewards may introduce noisy or overly diffuse credit signals. A similar pattern is observed for $\gamma$: with $\lambda=0.5$, the best performance is achieved at $\gamma=0.5$, while both smaller and larger values reduce performance. These results suggest that a moderate degree of future credit propagation provides the best trade-off between long-horizon reward assignment and optimization stability. \paragraph{Effect of environment scaling.} Finally, we study how the number of training environments affects generalization. Table~\ref{tab:env_scaling_ablation} reports results when varying the number of ToolVerse environments from 100 to 200, 300, and the full 422-environment setting. Increasing the number of environments consistently improves BFCL-v3, from 35.00\% with 100 environments to 37.50\% with the full set. The gains on $\bm{\tau^2}$-Bench are more pronounced when scaling from 200 to 300 environments, where the average score increases from 27.57\% to 32.12\%. The full environment set further improves performance to 32.37\%. These results support the importance of environment diversity for learning robust agentic tool-use behavior.

\clearpage
\section{Algorithm}

\begin{algorithm}[h]
\caption{Dynamic Unlocking Sampling}
\label{alg:dynamic_unlocking}
\begin{algorithmic}[1]
    \Require Tool dependency graph $\mathcal{G} = (\mathcal{V}, \mathcal{E})$, batch size $N$
    \Ensure Multi-turn trajectory $\mathcal{T}$

    \State $D[v] \gets \textsc{InDegree}(v)$ for all $v \in \mathcal{V}$
    \State $\mathcal{Q} \gets \{v \in \mathcal{V} \mid D[v] = 0\}$
    \State $\mathcal{T} \gets \emptyset$

    \While{$\mathcal{Q} \neq \emptyset$}
        \State $k \gets \min(|\mathcal{Q}|, N)$
        \State $\mathcal{S}_t \gets \textsc{Sample}(\mathcal{Q}, k)$
        \State $\mathcal{Q} \gets \mathcal{Q} \setminus \mathcal{S}_t$
        \State $\mathcal{T} \gets \mathcal{T} \cup \{\mathcal{S}_t\}$

        \ForAll{$u \in \mathcal{S}_t$}
            \ForAll{$v \in \textsc{Successors}(u, \mathcal{G})$}
                \State $D[v] \gets D[v] - 1$
                \If{$D[v] = 0$}
                    \State $\mathcal{Q} \gets \mathcal{Q} \cup \{v\}$
                \EndIf
            \EndFor
        \EndFor
    \EndWhile

    \State \Return $\mathcal{T}$
\end{algorithmic}
\end{algorithm}


\section{Case Study}
As show in Figure \ref{fig:case_study} , in our designed long-horizon tool use task, the agent faces the challenge of temporal dependency resolution. The model must not only plan the trajectory of tool invocations but also maintain a coherent internal state across multiple turns. As illustrated, the agent successfully identifies missing information (e.g., \texttt{budget constraints}), solicits it from the user, and synthesizes this new constraint with previously retrieved data (e.g., the specific date and location from the \texttt{Search\_Event} output) to execute the final \texttt{Book\_Hotel} action accurately.

\begin{figure}[h]
    \centering  
    \includegraphics[width=1\linewidth]{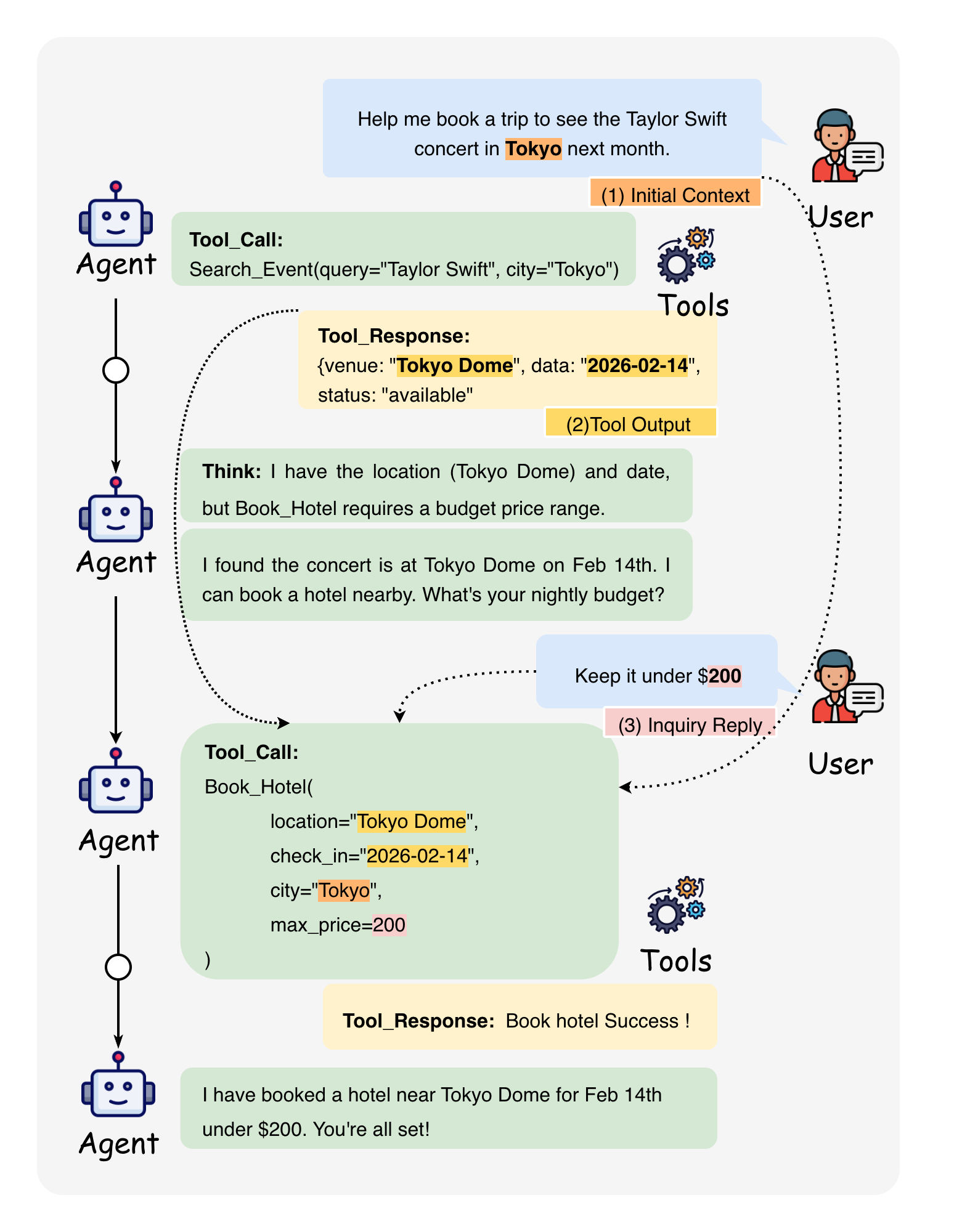}
    \caption{Long-horizon Multi-turn Tool use case Study. The agent demonstrates multi-turn contextual reasoning by inferring user intent, formulating a sequential execution plan, and dynamically resolving tool parameters (e.g., venue location and date) from prior dialogue history to complete the booking task.}
    \label{fig:case_study}
    \vskip -0.3in
\end{figure}

\newpage
\section{Prompts}

\subsection{Prompt Template for Task Generation}
\label{app:prompt_template}

Prompt: Appendix demonstrates the prompt template used to synthesize the user query and multi-turn dialogue based on the generated tool trajectory.

\begin{promptbox}{Prompt: Long-horizon Task Synthesis Prompt Template}
    \textbf{Role:} You are an expert AI dataset synthesizer specializing in constructing long-horizon reasoning tasks.\\
    
    \textbf{Input Data:}
    \begin{itemize}
        \item \textbf{Tool Dependency Graph:} A logical graph defining the causal links between tools.
        \item \textbf{Golden Trace:} A sequence of executed tool calls: $\mathcal{T} = [t_1, t_2, \dots, t_n]$, including input arguments and return values.
    \end{itemize}
    
    \textbf{Task:}
    Based on the provided \textit{Golden Trace} and dependency logic, generate a natural language \textbf{User Query} and a \textbf{Multi-turn Dialogue History} that would naturally lead to this specific sequence of tool executions.\\
    
    \textbf{Requirements:}
    \begin{enumerate}
        \item \textbf{Implicit Dependencies:} The user query should not explicitly list every step. Instead, it should state a high-level goal that implicitly requires the agent to figure out the dependencies (e.g., "Book me a flight" implies checking availability first).
        \item \textbf{Information Gap:} Ensure the dialogue reflects a realistic interaction where the agent might need to ask for clarification if parameters were missing in the early stages, matching the structure of the Golden Trace.
        \item \textbf{Consistency:} The generated user intent must strictly align with the final output of the last tool in the trace.
    \end{enumerate}
    
    \textbf{Input Trace:} \\
    \texttt{\{\{golden\_trace\_json\}\}} \\
    
    \textbf{Output Format:} \\
    Return the result in JSON format containing "user\_query", "dialogue\_history", and "reasoning\_chain".
\end{promptbox}

\begin{promptbox}{Prompt: Reproduce Tool Defination Prompt Template}

    You are now a JSON Tool Refactoring Assistant. Your core responsibility is to refactor tool definitions based on a given seed, following specified steps to output tool descriptions in JSON format that comply with the OpenAI Schema specification.

    \textbf{Core Analysis Steps (Step 1):}
    \begin{itemize}[leftmargin=*, noitemsep, topsep=2pt]
        \item Comprehensive Deconstruction: Understand all input tool information, including tool names, descriptions, parameters (name, type, description), and required fields.
        \item Business Scenario Analysis: Clarify the core business problems these tools solve and the applicable business domains.
        \item Compatibility Assessment: Judge whether the current tool definitions are suitable for implementation via Python functions with a dictionary database. Analyze issues in naming, parameter design, and functional descriptions (e.g., redundancy, parameter-function mismatch, or inability to implement independently).
    \end{itemize}
    
    \textbf{Refactoring Requirements (Step 2):}
    Refactor the names, descriptions, and parameters based on Step 1:
    \begin{enumerate}[leftmargin=*, noitemsep, topsep=2pt]
        \item Functional Fitness: Each refactored tool must be independently implementable by a single Python function with clear logic and no redundant cross-function dependencies.
        \item OpenAI Schema Refinement: 
        \begin{itemize}[leftmargin=*, noitemsep]
            \item Name: Concise and precise, reflecting the core function (lowercase + underscores).
            \item Description: Clear and complete, explaining the purpose and scope within the business scenario.
            \item Parameters: Optimize names and descriptions to ensure clarity; limit to a maximum of 4 parameters using standard Python types (int, float, str, list/array, dict).
            \item Call Correlation: Build business logic-driven constraints for the toolset. Define dependencies where the input of one tool correlates to the output of another, supporting multi-turn reasoning.
            \item Response Field: Add a mandatory "response" field containing at least "status" (success/error), "message", and other function-related output keys.
        \end{itemize}
    \end{enumerate}

    \textbf{Output Requirements:}
    \begin{itemize}[leftmargin=*, noitemsep, topsep=2pt]
        \item Scenario Description: Describe the analyzed result in one paragraph wrapped in \texttt{<description></description>} tags.
        \item Tool Output: A strict JSON array of objects with "type": "function" and "function" keys (containing name, description, parameters, and the custom response field).
        \item Language: The entire output must be in English.
    \end{itemize}

    \textbf{Input Tool Definition:} \\
    \texttt{\{\{tools\_defination\}\}}

\end{promptbox}

\end{document}